\documentclass[acmsmall, screen]{acmart}
\usepackage{amsmath}
\usepackage{hhline}
\usepackage{multirow}
\usepackage{bbm}
\usepackage{makecell}
\usepackage{amsfonts}
\usepackage{booktabs}
\usepackage{mathtools,xcolor}
\usepackage{algorithm}
\usepackage{algorithmic}
\usepackage{float}
\usepackage{hyperref}
\usepackage{amsthm}
\newtheorem{theorem}{Theorem}
\newcommand{\ariel}{{{\textsc{ArieL}}}}
\AtBeginDocument{%
  \providecommand\BibTeX{{%
    \normalfont B\kern-0.5em{\scshape i\kern-0.25em b}\kern-0.8em\TeX}}}

\setcopyright{rightsretained}
\acmJournal{TKDD}
\acmYear{2023} \acmVolume{X} \acmNumber{Y} \acmArticle{1} \acmMonth{12} \acmPrice{}
\acmDOI{10.1145/3638054}

\begin{document}

\title{\ariel: Adversarial Graph Contrastive Learning}
\titlenote{This is an extended version of our conference paper at ACM Web Conference 2022: Adversarial Graph Contrastive Learning with Information Regularization \cite{https://doi.org/10.48550/arxiv.2202.06491}.} 

\author{Shengyu Feng}
\orcid{0000-0002-2464-7226}
\affiliation{
  \streetaddress{5000 Forbes Avenue}
  \institution{Carnegie Mellon University}
  \city{Pittaburgh}
  \state{Pennsylvania}
  \country{USA}
  \postcode{15213-8213}
}
\email{shengyuf@andrew.cmu.edu}

\author{Baoyu Jing}
\orcid{0000-0003-1564-6499}
\affiliation{
  \streetaddress{201 North Goodwin Avenue}
  \institution{University of Illinois at Urbana-Champaign}
  \city{Urbana}
  \state{Illinois}
  \country{USA}
  \postcode{61801-2302}
}
\email{baoyuj2@illinois.edu}

\author{Yada Zhu}
\orcid{0000-0002-3338-6371}
\email{yzhu@us.ibm.com}
\affiliation{
  \institution{IBM Research}
  \streetaddress{1101 Kitchawan Road}
  \city{Yorktown Heights}
  \state{New York}
  \country{USA}
  \postcode{10562-1301}
}

\author{Hanghang Tong}
\orcid{0000-0003-4405-3887}
\affiliation{%
  \institution{University of Illinois at Urbana-Champaign}
  \streetaddress{201 North Goodwin Avenue}
  \city{Urbana}
  \state{Illinois}
  \country{USA}
  \postcode{61801-2302}
}
\email{htong@illinois.edu}

\begin{abstract}
Contrastive learning is an effective unsupervised method in graph representation learning, and the key component of contrastive learning lies in the construction of positive and negative samples. Previous methods usually utilize the proximity of nodes in the graph as the principle. Recently, the data-augmentation-based contrastive learning method has advanced to show great power in the visual domain, and some works extended this method from images to graphs. However, unlike the data augmentation on images, the data augmentation on graphs is far less intuitive and much harder to provide high-quality contrastive samples, which leaves much space for improvement. In this work, by introducing an adversarial graph view for data augmentation, we propose a simple but effective method, \textit{Adversarial Graph Contrastive Learning} (\ariel), to extract informative contrastive
samples within reasonable constraints. We develop a new technique called information regularization for stable training and use subgraph sampling for scalability. We generalize our method from node-level contrastive learning to the graph level by treating each graph instance as a super-node. \ariel\ consistently outperforms the current graph contrastive learning methods for both node-level and graph-level classification tasks on real-world datasets. We further demonstrate that \ariel\ is more robust in the face of adversarial attacks.
\end{abstract}



\maketitle

\section{Introduction}
Contrastive learning is a widely used technique in various graph representation learning tasks. In contrastive learning, the model tries to minimize the distances among positive pairs and maximize the distances among negative pairs in the embedding space \cite{velickovic2018deep, you2020graph, hassani2020contrastive, Zhu_2021, li2022graph, https://doi.org/10.48550/arxiv.2109.03560,jing2022coin,jing2023sterling,yan2023reconciling,wang2023characterizing}. The definition of positive and negative pairs is the key component in contrastive learning. Earlier methods like DeepWalk \cite{Perozzi:2014:DOL:2623330.2623732} and node2vec \cite{grover2016node2vec} define positive and negative pairs based on the co-occurrence of node pairs in the random walks. For knowledge graph embedding, it is a common practice to define positive and negative pairs based on translations \cite{NIPS2013_1cecc7a7, 10.5555/2893873.2894046, 10.5555/2886521.2886624, ji-etal-2015-knowledge, yan2021dynamic, wang2018acekg, Rossi_2021}.   

Recently, the breakthroughs of contrastive learning in computer vision have inspired some works to apply similar ideas from visual representation learning to graph representation learning.
To name a few, Deep Graph Infomax (DGI) \cite{velickovic2018deep} extends Deep InfoMax \cite{hjelm2019learning} and achieves significant improvements over previous random-walk-based methods. Graphical Mutual Information (GMI) \cite{peng2020graph} uses the same framework as DGI but generalizes the concept of mutual information from vector space to graph domain. Contrastive multi-view graph representation learning (referred to as MVGRL in this paper) \cite{hassani2020contrastive} further improves DGI by introducing graph diffusion into the contrastive learning framework.  The more recent works often follow the data-augmentation-based contrastive learning methods \cite{he2020momentum, chen2020simple}, which treat the data-augmented samples from the same instance as positive pairs and different instances as negative pairs. Graph Contrastive Coding (GCC) \cite{Qiu_2020} uses random walks with restart \cite{tong2006fast} to generate two subgraphs for each node as two data-augmented samples.
Graph Contrastive learning with Adaptive augmentation (GCA) \cite{Zhu_2021} introduces an adaptive data augmentation method that perturbs both the node features and edges according to their importance, and it is trained in a similar way as the famous visual contrastive learning framework SimCLR \cite{chen2020simple}. Its preliminary work, which uses uniform random sampling rather than adaptive sampling, is referred to as GRACE \cite{zhu2020deep} in this paper.
Robinson et al. \cite{robinson2021contrastive} propose a way to select hard negative samples based on the distances in the embedding space, and use it to obtain high-quality graph embedding. There are also many works \cite{you2020graph, zhao2020data} systemically studying the data augmentation on the graphs.
 
However, unlike the rotation and color jitter operations on images, the transformations on graphs, such as edge dropping and feature masking, are far less intuitive to human beings. The data augmentation on the graph could be either too similar to or totally different from the original graph. This, in turn, leads to a crucial question, that is, {\em how to generate a new graph that is hard enough for the model to discriminate from the original one, and in the meanwhile also maintains the desired properties?} 

Inspired by some recent works
\cite{kim2020adversarial, jiang2020robust, ho2020contrastive, jovanovi2021robust, suresh2021adversarial}, 
we introduce the adversarial training on the graph contrastive learning and propose a new framework called \textit{\underline{A}dversarial G\underline{R}aph Contrast\underline{I}v\underline{E} \underline{L}earning} (\ariel). Through the adversarial attack on both topology and node features, we generate an adversarial sample from the original graph. On the one hand, since the perturbation is under the constraint, the adversarial sample still stays close enough to the original one. On the other hand, the adversarial attack makes sure the adversarial sample is hard to discriminate from the other view by increasing the contrastive loss. On top of that, we propose a new constraint
called information regularization which could stabilize the training of \ariel\ and prevent the collapsing. We bridge the gap between node-level graph contrastive learning and graph-level contrastive learning by treating each graph instance as a super-node in node-level graph contrastive learning and thus make \ariel\ a universal graph representation learning framework. We demonstrate that the proposed \ariel\  outperforms the existing graph contrastive learning frameworks in the node classification and graph classification tasks on both real-world graphs and adversarially attacked graphs. 

In summary, we make the following contributions.

First, we introduce an adversarial view as a new form of data augmentation in graph contrastive learning, which makes the data augmentation more informative under mild perturbations.

Second, we propose a new technique called information regularization to stabilize the training of adversarial graph contrastive learning by regularizing the mutual information among positive pairs. 

Furthermore, we bridge the gap between node-level graph contrastive learning and graph-level contrastive learning and we unify their formulation under our framework.

Finally, we empirically demonstrate that \ariel\ can achieve better performance and higher robustness compared with previous graph contrastive learning methods.

The rest of the paper is organized as follows. Section 2 gives the problem definition of graph representation learning and the preliminaries. Section 3 describes the proposed algorithm. The experimental results are presented in section 4. After reviewing related work in section 5, we conclude the paper in section 6.

\section{Problem Definition}
In this section, we will introduce all the notations used in this paper and give a formal definition of our problem. Besides, we briefly introduce the preliminaries of our method.

\subsection{Graph Representation Learning}
For graph representation learning, let $G=\{\mathcal{V},\mathcal{E},\mathbf{X}\}$ be an attributed graph, where $\mathcal{V}=\{v_1,v_2,\dots,v_n\}$ denotes the set of nodes, $\mathcal{E}\subseteq \mathcal{V}\times \mathcal{V}$ denotes the set of edges, and $\mathbf{X}\in\mathbb{R}^{n\times d}$ denotes the feature matrix. Each node $v_i$ has a $d$-dimensional feature $\mathbf{X}[i,:]$, and all edges are assumed to be unweighted and undirected. We use a binary adjacency matrix $\mathbf{A}\in \{0,1\}^{n\times n}$ to represent the information of nodes and edges, where $\mathbf{A}[i,j]=1$ if and only if the node pair $(v_i,v_j) \in \mathcal{E}$. In the following text, we will use $G=\{\mathbf{A},\mathbf{X}\}$ to represent the graph. 

The objective of the graph representation learning is to learn an encoder $f: \mathbb{R}^{n\times n}\times\mathbb{R}^{n\times d}\rightarrow\mathbb{R}^{n\times d'}$, which maps the nodes in the graph into low-dimensional embeddings. Denote the node embedding matrix $\mathbf{H}=f(\mathbf{A},\mathbf{X})$, where $\mathbf{H}[i,:]\in \mathbb{R}^{d'}$ is the embedding for node $v_i$. This representation could be used for downstream tasks like node classification. Based on the node embedding matrix, we can further obtain the graph embedding through an order-invariant readout function $R(\cdot)$, which generates the graph representation as 
$R(\mathbf{H})\in\mathbb{R}^{d''}$.

\subsection{InfoNCE Loss}
InfoNCE loss \cite{oord2019representation} is the predominant work-horse of the contrastive learning loss, which maximizes the lower bound of the mutual information between two random variables. For each positive pair $(\mathbf{x},\mathbf{x}^+)$ associated with $k$ negative samples of $\mathbf{x}$, denoted as $\{\mathbf{x}_1^-, \mathbf{x}_2^-, \cdots,\mathbf{x}_k^-\}$, InfoNCE loss could be written as
\begin{align}
    L_k = -\log{(\frac{g(\mathbf{x},\mathbf{x}^+)}{g(\mathbf{x},\mathbf{x}^+)+\sum\limits_{i=1}^kg(\mathbf{x},\mathbf{x}_i^-)})}.
\end{align}
Here $g(\cdot)$ is the density ratio with the property that $ g(\mathbf{a},\mathbf{b}) \propto \frac{p(\mathbf{a}|\mathbf{b})}{p(\mathbf{a})}$, 
where $\propto$ stands for \textit{proportional to}. It has been shown by \cite{oord2019representation} that $-L_k$ actually serves as the lower bound of the mutual information $I(\mathbf{x};\mathbf{x}^+)$ with
\begin{align}
     I(\mathbf{x};\mathbf{x}^+)\geq \log{(k)}-L_k.
\end{align}
\subsection{Graph Contrastive Learning}
We build the proposed method upon the framework of SimCLR \cite{chen2020simple},
which is also the basic framework that GCA \cite{Zhu_2021} and GraphCL \cite{you2020graph} are built on.
\subsubsection{Node-level Contrastive Learning}
Given a graph $G$, two views of the graph $G_1=\{\mathbf{A}_1,\mathbf{X}_1\}$ and $G_2=\{\mathbf{A}_2,\mathbf{X}_2\}$ are first generated. This step can be treated as the data augmentation on the original graph, and various augmentation methods can be used herein. We use random edge dropping and feature masking as GCA does. The node embedding matrix for each graph can be computed as $\mathbf{H}_1=f(\mathbf{A}_1,\mathbf{X}_1)$ and $\mathbf{H}_2=f(\mathbf{A}_2,\mathbf{X}_2)$. The corresponding node pairs in two graph views are the positive pairs and all other node pairs are negative. Define $\theta(\mathbf{u},\mathbf{v})$ to be the similarity function between vectors $\mathbf{u}$ and $\mathbf{v}$, in practice, it is usually chosen as the cosine similarity on the projected embedding of each vector, using a two-layer neural network as the projection head. Denote $\mathbf{u}_i=\mathbf{H}_1[i,:]$ and $\mathbf{v}_i=\mathbf{H}_2[i,:]$, the contrastive loss is defined as
\begin{align}
\label{eq:contra}
L_{\text{con}}(G_1,G_2) &= \frac{1}{2n}\sum_{i=1}^n(l(\mathbf{u}_i,\mathbf{v}_i)+l(\mathbf{v}_i,\mathbf{u}_i)),\\
l(\mathbf{u}_i,\mathbf{v}_i) &=
-\log{\frac{e^{\theta(\mathbf{u}_i,\mathbf{v}_i)/\tau}}{e^{\theta(\mathbf{u}_i,\mathbf{v}_i)/\tau} +\sum\limits_{j\neq i}e^{\theta(\mathbf{u}_i,\mathbf{v}_j)/\tau}+\sum\limits_{j\neq i}e^{\theta(\mathbf{u}_i,\mathbf{u}_j)/\tau}}},
\end{align}
where $\tau$ is a temperature parameter.
$l(\mathbf{v}_i,\mathbf{u}_i)$ is symmetrically defined by exchanging the variables in $l(\mathbf{u}_i,\mathbf{v}_i)$. This loss is basically a variant of InfoNCE loss which is symmetrically defined instead.

\subsubsection{Graph-level Contrastive Learning}
The graph-level contrastive learning is closer to contrastive learning in the visual domain. For a batch of graphs $\mathcal{B}=\{G_1,\cdots, G_b\}$, we obtain the augmentation of each graph as $\mathcal{B}^+=\{G_1^+,\cdots,G_b^+\}$ through node dropping, subgraph sampling, edge perturbation, and feature masking as in GraphCL \cite{you2020graph}, the loss function is thus defined on these two batches of graphs as  
\begin{equation}
    L_{\text{con}}(\mathcal{B},\mathcal{B}^+) = \mathbb{E}[-\log{\frac{e^{\theta(\mathbf{R}_i, \mathbf{R}_i^+)/\tau}}{e^{\theta(\mathbf{R}_i, \mathbf{R}_i^+)/\tau}+\sum\limits_{j\neq i}e^{\theta(\mathbf{R}_i, \mathbf{R}_j^+)/\tau}+\sum\limits_{j\neq i}e^{\theta(\mathbf{R}_i, \mathbf{R}_j)/\tau}}}],
\end{equation}
where $\mathbf{R}_i = R(\mathbf{H}_i)$ and $\mathbf{R}_i^+ = R(\mathbf{H}_i^+)$.

By abuse of notation, we also use $L_{\text{con}}$ to denote the loss function for the graph-level contrastive learning, the actual meaning of $L_{\text{con}}$ is dependent on the input type, graph or set, in the following text. 

Specifically, we notice that a set of graphs with $G_i=\{\mathbf{A}_i, \mathbf{X}_i\}$ can be combined into one graph as 
\begin{equation}
G^*=\{\text{block\_diag}(\mathbf{A}_1,\cdots,\mathbf{A}_b), \text{Concat}(\mathbf{X}_1, \cdots,\mathbf{X}_b)\}.
\end{equation}
Under this transformation, graph embedding of $G_i$ can be treated as the embedding of a super-node in $G^*$. This observation helps us bridge the gap between node-level contrastive learning and graph-level contrastive learning, where the only difference between them is the granularity of the instance in the contrastive learning loss. Therefore, we can build a universal framework for graph contrastive learning which can be used for both node-level and graph-level downstream tasks. 

\subsubsection{Graph Encoder}
In principle, our framework could be applied on any graph neural network (GNN) architecture for encoder as long as it could be attacked. For simplicity,
we employ a two-layer Graph Convolutional Network (GCN) \cite{kipf2017semisupervised} for node-level contrastive learning and a three-layer Graph Isomorphism Network (GIN) \cite{https://doi.org/10.48550/arxiv.1810.00826} for graph-level contrastive learning in this work. Define the symmetrically normalized adjacency matrix
\begin{align}
\hat{\mathbf{A}} =  \Tilde{\mathbf{D}}^{-\frac{1}{2}}\Tilde{\mathbf{A}}\Tilde{\mathbf{D}}^{-\frac{1}{2}}, 
\end{align}
where $\Tilde{\mathbf{A}} = \mathbf{A} +\mathbf{I}_n$ is the adjacency matrix with self-connections added and $\mathbf{I}_n$ is the identity matrix, $\Tilde{\mathbf{D}}$ is the diagonal degree matrix of $\Tilde{\mathbf{A}}$ with $\Tilde{\mathbf{D}}[i,i] = \sum_j\Tilde{\mathbf{A}}[i,j]$.   
The two-layer GCN is given as 
\begin{align}
    f(\mathbf{A},\mathbf{X}) = \sigma(\hat{\mathbf{A}}\sigma(\hat{\mathbf{A}}\mathbf{X}\mathbf{W}^{(1)})\mathbf{W}^{(2)}),
\end{align}
where $\mathbf{W}^{(1)}$ and $\mathbf{W}^{(2)}$ are the weights of the first and second layer respectively, $\sigma(\cdot)$ is the activation function.   

The Graph Isomorphism operator could be defined as 
\begin{align}
    \mathbf{X}' = h((\mathbf{A}+\epsilon \mathbf{I})\mathbf{X}),
\end{align}
where $h(\cdot)$ is a neural network such as multi-layer perceptrons (MLPs) and $\epsilon$ is a non-negative scalar. A three-layer GIN is the stack of three Graph Isomorphism operators. In this work, $h(\cdot)$ is a two-layer MLP followed by an activation function and Batch Normalization \cite{https://doi.org/10.48550/arxiv.1502.03167}, and $\epsilon$ is set as $0$ for all operators. Use $\mathbf{X}^{(i)}$ to denote the node embeddings after the $i$-th operator, the final node embeddings are the concatenation of $\mathbf{X}^{(i)}$, $\mathbf{H}=\text{Concat}(\mathbf{X}^{(i)}|i=1,2,3)$, and the graph embedding is the concatenation of the node embeddings after mean-pooling, $R(\mathbf{H})=\text{Concat}(\text{Mean}(\mathbf{X}^{(i)})|i=1,2,3)$.

\subsection{Projected Gradient Descent Attack}
Projected Gradient Descent (PGD) attack \cite{madry2019deep} is an iterative attack method that projects the perturbation onto the ball of interest at the end of each iteration. Assuming that the loss $L(\cdot)$ is a function of the input matrix $\mathbf{Z} \in \mathbb{R}^{n\times d}$, at $t$-th iteration, the perturbation matrix $\mathbf{\Delta}_t\in\mathbb{R}^{n\times d}$ under an $l_{\infty}$-norm constraint could be written as
\begin{align}
    \mathbf{\Delta}_{t} = \Pi_{\|\mathbf{\Delta}\|_{\infty}\leq \delta} (\mathbf{\Delta}_{t-1}+ \eta\cdot \text{sgn}(\nabla_{\mathbf{\Delta}}L(\mathbf{Z}+\mathbf{\Delta}_{t-1})),
\end{align}
where $\eta$ is the step size, $\text{sgn}(\cdot)$ takes the sign of each element in the input matrix, and $\Pi_{\|\mathbf{\Delta}\|_{\infty}\leq \delta}$ projects the perturbation onto the $\delta$-ball in the $l_{\infty}$-norm. 

\section{Method}
In this section, we will first investigate the vulnerability of the graph contrastive learning, then we will spend the remaining section discussing each part of \ariel\ in detail. Based on the connection we build upon the node-level contrastive learning and graph-level contrastive learning, we will illustrate our method from the perspective of node-level contrastive learning and extend it to the graph level.
\subsection{Vulnerability of the Graph Contrastive Learning}
Many GNNs are known to be vulnerable to adversarial attacks \cite{bojchevski2019adversarial, Z_gner_2018}, so we first investigate the vulnerability of the graph neural networks trained with the contrastive learning objective in Equation (\ref{eq:contra}). We generate a sequence of $60$ 
graphs by iteratively dropping edges and masking the features. Let $G_0=G$, for the $t$-th iteration, we generate $G_t$ from $G_{t-1}$ by randomly dropping the edges in $G_{t-1}$ and randomly masking the unmasked features, both with probability $p=0.03$. 
Since $G_t$ is guaranteed to contain less information than $G_{t-1}$, $G_t$ should be less similar to $G_0$ than $G_{t-1}$, on both the graph and node level. Denote the node embeddings of $G_t$ as $\mathbf{H}_t$, we measure the similarity $\theta(\mathbf{H}_t[i,:], \mathbf{H}_0[i,:])$, and it is expected that the similarity decreases as the iteration goes on. 

We generate the sequences on two datasets, \textit{Amazon-Computers} and \textit{Amazon-Photo} \cite{shchur2019pitfalls}, and the results are shown in Figure \ref{fig:simi1}.
\begin{figure}[h]
    \centering
    \includegraphics[width=.5\linewidth]{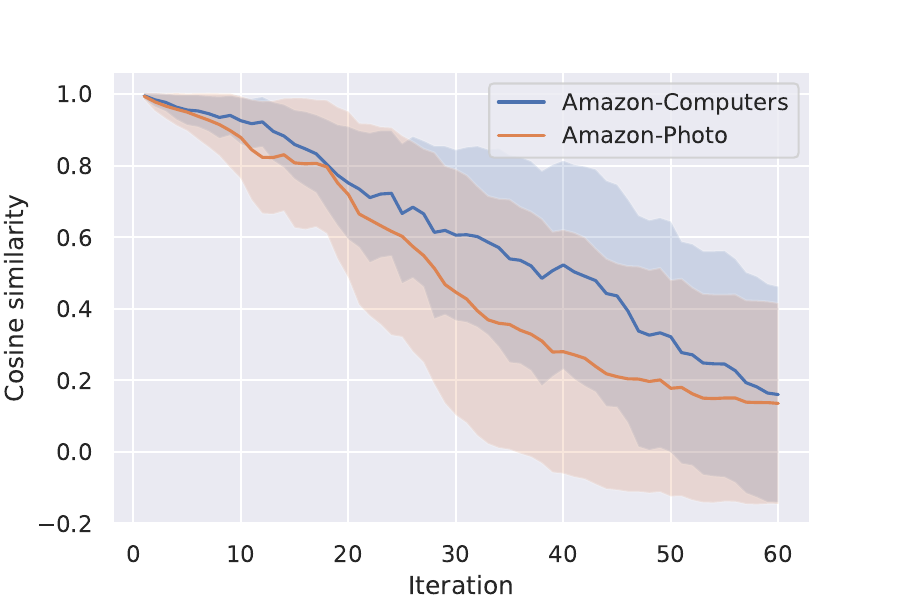}
    \caption{Average cosine similarity between the node embeddings of the original graph and the perturbed graph, results are on datasets Amazon-Computers and Amazon-Photo. The shaded area represents the standard deviation.}
    \label{fig:simi1}
\end{figure}
At the $30$-th iteration, with $0.97^{30}=40.10\%$ edges and features left, the average similarity of the positive samples are under $0.5$ on Amazon-Photo. At the $60$-th iteration, with $0.97^{60}=16.08\%$ edges and features left, the average similarity drops under $0.2$ on both Amazon-Computers and Amazon-Photo. Additionally, starting from the $30$-th iteration, the cosine similarity has around $0.3$ standard deviation for both datasets, which indicates that a lot of nodes are actually very sensitive to the external perturbations, even if we do not add any adversarial component but just mask out some information. These results demonstrate that the current graph contrastive learning framework is not trained over enough high-quality contrastive samples and is not robust to adversarial attacks.

Given this observation, we are motivated to build an adversarial graph contrastive learning framework that could improve the performance and robustness of the previous graph contrastive learning methods. The overview of our framework is shown in Figure \ref{fig:architecture}.
\begin{figure*}
    \centering
    \includegraphics[width=0.95\textwidth]{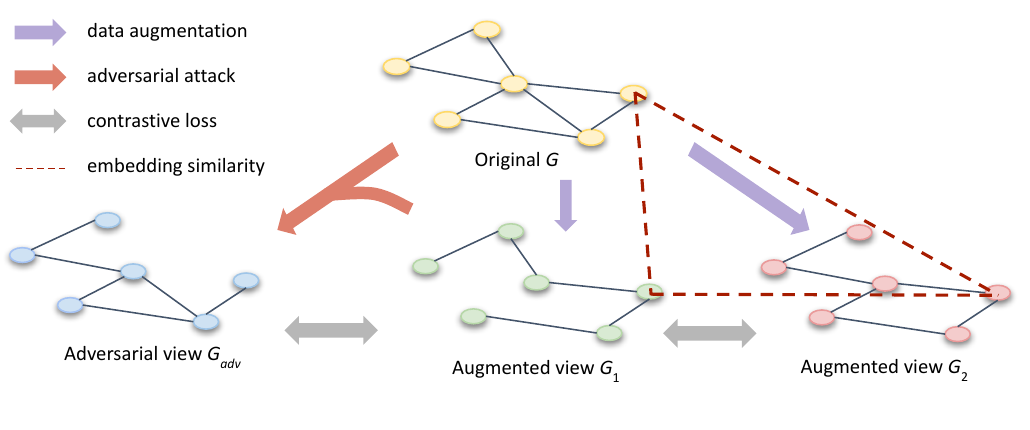}
    \caption{The overview of the proposed \ariel\ framework. For each iteration, two augmented views are generated from the original graph by data augmentation (purple arrows), and then an adversarial view is generated (red arrow) from the original graph by maximizing the contrastive loss against one of the augmented views. Besides, the similarities of the corresponding nodes (dashed lines) will get penalized by the information regularization if they exceed the estimated upper bound. The objective of \ariel\ is to minimize the contrastive loss (grey arrows) between the augmented views, the adversarial view, and the corresponding augmented view, and the information regularization. Best viewed in color.}
    \label{fig:architecture}
\vspace{-10pt}
\end{figure*}
\subsection{Adversarial Training}
Adversarial training uses the samples generated through the adversarial attack methods to improve the generalization ability and robustness of the original method during training. Although most existing attack frameworks are targeted at supervised learning, it is natural to generalize these methods to contrastive learning by replacing the classification loss with the contrastive loss. The goal of the adversarial attack on graph contrastive learning is to maximize the contrastive loss by adding a small perturbation on the contrastive samples, which can be formulated as 
\begin{align}
    \label{eq:adv_sample}
    G_{\text{adv}} = \arg\max_{G'} L_{\text{con}}(G_1, G'), 
\end{align}
where $G'=\{\mathbf{A}', \mathbf{X}'\}$ is generated from the original graph $G$, and the change is constrained by the budget $\Delta_{\mathbf{A}}$ and $\Delta_{\mathbf{X}}$ as
\begin{align}
    &\sum_{i,j}|\mathbf{A}'[i,j]-\mathbf{A}[i,j]|\leq\Delta_{\mathbf{A}},\\
    &\sum_{i,j}|\mathbf{X}'[i,j]-\mathbf{X}[i,j]|\leq\Delta_{\mathbf{X}}.
\end{align}
We treat adversarial attacks as one kind of data augmentation. Although we find it effective to make the adversarial attack on one or two augmented views as well, we follow the typical contrastive learning procedure as in SimCLR \cite{chen2020simple} to make the attack on the original graph in this work. Besides, it does not matter whether $G_1$, $G_2$ or $G$ is chosen as the anchor for the adversary, each choice works in our framework and it can also be sampled as a third view. In our experiments, we use PGD attack \cite{madry2019deep} as our attack method.

We generally follow the method proposed by Xu et al. \cite{xu2019topology} to make the PGD attack on the graph structure and apply the regular PGD attack method on the node features. Define the supplement of the adjacency matrix as $\mathbf{\bar{A}}=\mathbf{1}_{n\times n}-\mathbf{I}_n-\mathbf{A}$, where $\mathbf{1}_{n\times n}$ is the ones matrix of size $n\times n$. The perturbed adjacency matrix can be written as
\begin{align}
    \mathbf{A}_{\text{adv}} &= \mathbf{A}+ \mathbf{C}\circ \mathbf{L}_{\mathbf{A}},\\
    \mathbf{C} &= \bar{\mathbf{A}} - \mathbf{A},
\end{align}
where $\circ$ is the element-wise product, and $\mathbf{L}_{\mathbf{A}}\in\{0,1\}^{n\times n}$ is a symmetric matrix with each element $\mathbf{L}_{\mathbf{A}}[i,j]$ corresponding to the modification (e.g., add, delete or no modification) of the edge between the node pair $(v_i,v_j)$. The perturbation on $\mathbf{X}$ follows the regular PGD attack procedure and the perturbed feature matrix can be written as 
\begin{align}
    \mathbf{X}_{\text{adv}} = \mathbf{X} + \mathbf{L}_{\mathbf{X}},
\end{align}
where $\mathbf{L}_{\mathbf{X}}\in\mathbb{R}^{n\times d}$ is the perturbation on the feature matrix.

For the ease of optimization, $\mathbf{L}_{\mathbf{A}}$ is relaxed to its convex hull $\Tilde{\mathbf{L}}_{\mathbf{A}}\in[0,1]^{n\times n}$, which satisfies $\mathcal{S}_{\mathbf{A}}=\{\Tilde{\mathbf{L}}_{\mathbf{A}}|\sum\limits_{i,j}\Tilde{\mathbf{L}}_{\mathbf{A}}\leq \Delta_{\mathbf{A}}, \Tilde{\mathbf{L}}_{\mathbf{A}}\in[0,1]^{n\times n}\}$. The constraint on $\mathbf{L}_{\mathbf{X}}$ can be written as $\mathcal{S}_{\mathbf{X}}=\{\mathbf{L}_{\mathbf{X}}|\| \mathbf{L}_{\mathbf{X}}\|_{\infty}\leq\delta_{\mathbf{X}}, \mathbf{L}_{\mathbf{X}}\in\mathbb{R}^{n\times d}\}$, where we directly treat $\delta_{\mathbf{X}}$ as the constraint on the feature perturbation. In each iteration, we make the updates
\begin{align}
    \label{eq:pgd_a}
    \Tilde{\mathbf{L}}_{\mathbf{A}}^{(t)} &= \Pi_{\mathcal{S}_{\mathbf{A}}}[\Tilde{\mathbf{L}}_{\mathbf{A}}^{(t-1)}+\alpha\cdot\mathbf{G}_{\mathbf{A}}^{(t)}],\\
    \label{eq:pgd_x}
    \mathbf{L}_{\mathbf{X}}^{(t)} & =\Pi_{\mathcal{S}_{\mathbf{X}}}[\mathbf{L}_{\mathbf{X}}^{(t-1)}+\beta\cdot\text{sgn}(\mathbf{G}_{\mathbf{X}}^{(t)})],
\end{align}
where $t$ denotes the current number of iterations, and 
\begin{align}
\mathbf{G}_{\mathbf{A}}^{(t)}&=\nabla_{\Tilde{\mathbf{L}}_{\mathbf{A}}}   L_{\text{con}}(G_1,G_{\text{adv}}^{(t-1)}),\\ \mathbf{G}_{\mathbf{X}}^{(t)}&=\nabla_{\mathbf{L}_{\mathbf{X}}}  L_{\text{con}}(G_1,G_{\text{adv}}^{(t-1)}),
\end{align}
denote the gradients of the loss with respect to $\Tilde{\mathbf{L}}_{\mathbf{A}}$ at $\Tilde{\mathbf{L}}_{\mathbf{A}}^{(t-1)}$ and $\mathbf{L}_{\mathbf{X}}$ at $\mathbf{L}_{\mathbf{X}}^{(t-1)}$ respectively. Here $G_{\text{adv}}^{(t-1)}$ is defined as $\{\mathbf{A}+ \mathbf{C}\circ \Tilde{\mathbf{L}}_{\mathbf{A}}^{(t-1)},  \mathbf{X} + \mathbf{L}_{\mathbf{X}}^{(t-1)}\}$.
The projection operation $\Pi_{\mathcal{S}_{\mathbf{X}}}$  simply clips $\mathbf{L}_{\mathbf{X}}$ into the range $[-\delta_{\mathbf{X}}, \delta_{\mathbf{X}}]$ elementwisely. The projection operation  $\Pi_{\mathcal{S}_{\mathbf{A}}}$ is calculated as  
\begin{align}
\label{eq:proj}
    \Pi_{\mathcal{S}_{\mathbf{A}}}(\mathbf{Z}) = \begin{cases}
    P_{[0,1]}[\mathbf{Z}-\mu \mathbf{1}_{n\times n}], & \text{if } \mu>0, \text{and } \sum\limits_{i,j}{P_{[0,1]}[\mathbf{Z}-\mu \mathbf{1}_{n\times n}]}=\Delta_{\mathbf{A}},\\
    \\
     P_{[0,1]}[\mathbf{Z}], & \text{if } \sum\limits_{i,j}{P_{[0,1]}[\mathbf{Z}]}\leq\Delta_{\mathbf{A}},
    \end{cases}
\end{align}
where $P_{[0,1]}[\mathbf{Z}]$ clips $\mathbf{Z}$ into the range $[0,1]$. We use the bisection method \cite{boyd_vandenberghe_2004} to solve the equation $\sum\limits_{i,j} P_{[0,1]}[\mathbf{Z}-\mu \mathbf{1}_{n\times n}]=\Delta_{\mathbf{A}}$ with respect to the dual variable $\mu$.

To finally obtain $\mathbf{L}_{\mathbf{A}}$ from $\Tilde{\mathbf{L}}_{\mathbf{A}}$, each element is independently sampled from a Bernoulli distribution as $ \mathbf{L}_{\mathbf{A}}[i,j]\sim \text{Bernoulli}( \Tilde{\mathbf{L}}_{\mathbf{A}}[i,j])$. To obtain a symmetric matrix, we only sample the upper triangular part (the elements on the diagonal are known to be $0$ in our formulation) and obtain the lower triangular part through transposition.

\subsection{Adversarial Graph Contrastive Learning}
To assimilate the graph contrastive learning and adversarial training together, we treat the adversarial view $G_{\text{adv}}$ obtained from Equation (\ref{eq:adv_sample}) as another view of the graph. We define the adversarial contrastive loss as the contrastive loss between $G_1$ and $G_{\text{adv}}$. The adversarial contrastive loss is added to the original contrastive loss in Equation (\ref{eq:contra}), which becomes
\begin{align}\label{eq:obj}
L(G_1,G_2, G_{\text{adv}}) = L_{\text{con}}(G_1,G_2)+\epsilon_1 L_{\text{con}}(G_1,G_{\text{adv}}),
\end{align}
where $\epsilon_1 >0$ is the adversarial contrastive loss coefficient. We further adopt two additional subtleties on top of this basic framework: subgraph sampling and curriculum learning.  For each iteration, a subgraph $G_s$ with a fixed size is first sampled from the original graph $G$, then the data augmentation and adversarial attack are both conducted on this subgraph. The subgraph sampling could avoid the gradient derivation on the whole graph, which will lead to heavy computation on a large network. Besides, we also observe that subgraph sampling could increase the randomness of the sample and sometimes boost the performance. To avoid the imbalanced sample on the isolated nodes, we uniformly sample a random set of nodes and then construct the subgraph atop them. For every $T$ epochs, the adversarial contrastive loss coefficient is multiplied by a weight $\gamma$. When $\gamma>1$, the portion of the adversarial contrastive loss is gradually increasing and the contrastive learning becomes harder as the training goes on.

\subsection{Information Regularization}
Adversarial training could effectively improve the model's robustness to the perturbations, nonetheless, we find these hard training samples could impose the additional risk of training collapsing, i.e., the model will be located at a bad parameter area at the early stage of the training, assigning higher probability to a highly perturbed sample than a mildly perturbed one. In our experiment, we find this vanilla adversarial training method may fail to converge in some cases (e.g., Amazon-Photo dataset). To stabilize the training, we add one constraint termed {\em information regularization}, whose main goal is to regularize the instance similarity in the feature space.

The data processing  inequality \cite{Cover2006} states that for three random variables $\mathbf{Z}_1$, $\mathbf{Z}_2$ and $\mathbf{Z}_3\in\mathbb{R}^{n\times d'}$, if they satisfy the Markov relation $\mathbf{Z_1}\rightarrow\mathbf{Z}_2\rightarrow\mathbf{Z}_3$, then the inequality $I(\mathbf{Z}_1;\mathbf{Z}_3)\leq I(\mathbf{Z}_1;\mathbf{Z}_2)$ holds. As proved by Zhu et al. \cite{Zhu_2021}, since the node embeddings of two views  $\mathbf{H}_1$ and $\mathbf{H}_2$ are conditionally independent given the node embeddings of the original graph $\mathbf{H}$, they also satisfy the Markov relation with $\mathbf{H}_1\rightarrow\mathbf{H}\rightarrow\mathbf{H}_2$, and vice versa. Therefore, we can derive the following properties over their mutual information
\begin{align}
    I(\mathbf{H}_1; \mathbf{H}_2) &\leq I(\mathbf{H};\mathbf{H}_1), \\
    I(\mathbf{H}_1; \mathbf{H}_2) &\leq I(\mathbf{H};\mathbf{H}_2).
\end{align}
In fact, this inequality holds on each node. A sketch of the proof is that the embedding of each node $v_i$ is determined by all the nodes from its $l$-hop neighborhood if an $l$-layer GNN is used as the encoder, and this subgraph composed of its $l$-hop neighborhood also satisfies the Markov relation. Therefore, we can derive the more strict inequalities 
\begin{align}
\label{info_ineq1}
    I(\mathbf{H}_1[i,:]; \mathbf{H}_2[i,:]) &\leq I(\mathbf{H}[i,:];\mathbf{H}_1[i,:]), \\
\label{info_ineq2}
     I(\mathbf{H}_1[i,:]; \mathbf{H}_2[i,:]) &\leq I(\mathbf{H}[i,:];\mathbf{H}_2[i,:]).
\end{align}
Since $-L_{\text{con}}(G_1,G_2)$ is only a lower bound of the mutual information, directly applying the above constraints is hard, we only consider the constraints on the density ratio. Using the Markov relation for each node, we give the following theorem:
\begin{theorem}
For two graph views $G_1$ and $G_2$ independently transformed from the graph $G$, the density ratio of their node embeddings $\mathbf{H}_1$ and $\mathbf{H}_2$ should satisfy $g(\mathbf{H}_2[i,:], \mathbf{H}_1[i,:])\leq g(\mathbf{H}_2[i,:], \mathbf{H}[i,:])$ and $g(\mathbf{H}_1[i,:], \mathbf{H}_2[i,:])\leq g(\mathbf{H}_1[i,:], \mathbf{H}[i,:])$, where $\mathbf{H}$ is the node embeddings of the original graph.
\end{theorem}
\begin{proof}
Following the Markov relation of each node, we get
\begin{equation}
\begin{split}
    p(\mathbf{H}_2[i,:]|\mathbf{H}_1[i,:])&=p(\mathbf{H}_2[i,:]|\mathbf{H}[i,:])p(\mathbf{H}[i,:]|\mathbf{H}_1[i,:])\\
    &\leq p(\mathbf{H}_2[i,:]|\mathbf{H}[i,:]),
\end{split}
\end{equation}
and consequently
\begin{align}
    \frac{p(\mathbf{H}_2[i,:]|\mathbf{H}_1[i,:])}{p(\mathbf{H}_2[i,:])}\leq \frac{p(\mathbf{H}_2[i,:]|\mathbf{H}[i,:])}{p(\mathbf{H}_2[i,:])}.
\end{align}
Since $g(\mathbf{a},\mathbf{b}) \propto \frac{p(\mathbf{a}|\mathbf{b})}{p(\mathbf{a})}$, we get $g(\mathbf{H}_2[i,:], \mathbf{H}_1[i,:])\leq g(\mathbf{H}_2[i,:], \mathbf{H}[i,:])$. A similar proof applies to the other inequality.
\end{proof}
 Note that $g(\cdot,\cdot)$ is symmetric for the two inputs, we thus get two upper bounds for $g(\mathbf{H}_1[i,:], \mathbf{H}_2[i,:])$. According to the previous definition, $g(\mathbf{a},\mathbf{b})=e^{\theta(\mathbf{a},\mathbf{b})/\tau}$, we can simply replace $g(\cdot, \cdot)$ with $\theta(\cdot,\cdot)$ in the inequalities, then we combine these two upper bounds into one
\begin{align}
    2\cdot\theta(\mathbf{H}_1[i,:], \mathbf{H}_2[i,:]) \leq \theta(\mathbf{H}_2[i,:], \mathbf{H}[i,:])+\theta(\mathbf{H}_1[i,:], \mathbf{H}[i,:]).
\end{align}
This bound intuitively requires the similarity between $\mathbf{H}_1[i,:]$ and $\mathbf{H}_2[i,:]$ to be less than the similarity between $\mathbf{H}[i,:]$ and $\mathbf{H}_1[i,:]$ or $\mathbf{H}_2[i,:]$.  
Equipped with this upper bound, we define the following information regularization to penalize the higher probability of a less similar contrastive pair
\begin{align}
\label{eq:ir}
&d_i = 2\cdot\theta(\mathbf{H}_1[i,:], \mathbf{H}_2[i,:]) -(\theta(\mathbf{H}_2[i,:], \mathbf{H}[i,:])+\theta(\mathbf{H}_1[i,:], \mathbf{H}[i,:])),\\
    &L_{I}(G_1, G_2, G) = \frac{1}{n}\sum\limits_{i=1}^n{\max\{d_i, 0\}}.
\end{align}
Specifically, the information regularization could be defined over any three graphs that satisfy the Markov relation, but for our framework, to save the memory and time complexity, we avoid additional sampling and directly ground the information regularization on the existing graphs. It is also fine to apply the information regularization on $G$, $G_1$ and $G_{\text{adv}}$ or $G$, $G_2$ and $G_{\text{adv}}$.

The final loss of \ariel\ can be written as 
\begin{equation}
\label{eq:ARIEL}
     L(G_1,G_2, G_{\text{adv}}) = L_{\text{con}}(G_1,G_2)+\epsilon_1 L_{\text{con}}(G_1,G_{\text{adv}}) + \epsilon_2  L_{I}(G_1, G_2, G),
\end{equation}
where $\epsilon_2 >0$ controls the strength of the information regularization.

The entire algorithm of \ariel\ is summarized in Algorithm \ref{alg:algorithm}.
\begin{algorithm}[tb]
\caption{Algorithm of \ariel}
\label{alg:algorithm}
\begin{algorithmic}[0] 
\STATE \textbf{Input data:} Graph $G=(\mathbf{A},\mathbf{X})$
\STATE \textbf{Input parameters:} $\alpha$, $\beta$, $\Delta_{\mathbf{A}}$, $\delta_{\mathbf{X}}$, $\epsilon_1$, $\epsilon_2$, $\gamma$ and $T$
\STATE Randomly initialize the graph encoder $f$
\FOR{iteration  $k=0,1,\cdots$}
\STATE Sample a subgraph $G_s$ from $G$
\STATE Generate two views $G_1$ and $G_2$ from $G_s$
\STATE Generate the adversarial view $G_{\text{adv}}$ according to Equations (\ref{eq:pgd_x}) and (\ref{eq:pgd_a})
\STATE Update model $f$ to minimize $L(G_1,G_2, G_{\text{adv}})$ in Equation (\ref{eq:ARIEL})
\IF{$(k+1)\mod T =0$}
\STATE Update $\epsilon_1\leftarrow \gamma*\epsilon_1$
\ENDIF
\ENDFOR
\STATE \textbf{return:} Node embedding matrix  $\mathbf{H}=f(\mathbf{A},\mathbf{X})$ 
\end{algorithmic}
\end{algorithm}

\subsection{Extension to Graph-Level Contrastive Learning}
For a batch of graphs $\mathcal{B}$ and the batch of their augmentation views $\mathcal{B}^+$, we aim to generate a batch of adversarial views, which we denote as $\mathcal{B}_{\text{adv}}$. Denote the combined graph of each batch as $G^*$, ${G^+}^*$ and $G_{\text{adv}}^*$. The objective of the adversarial graph contrastive learning on the graph level can be formulated as
\begin{align}
    \mathcal{B}_{\text{adv}} = \arg&\max_{\mathcal{B}'} L_{\text{con}}(\mathcal{B}^+, \mathcal{B}'),\\
    \text{subject to } & \sum_{i,j}|{\mathbf{A}'}^*[i,j]-{\mathbf{A}}^*[i,j]|\leq \Delta_{\mathbf{A}},\\
                       &  \sum_{i,j}|{\mathbf{X}'}^*[i,j]-{\mathbf{X}}^*[i,j]|\leq \Delta_{\mathbf{X}}.
\end{align}
It is worth noting that the constraints we use here are applied on the batch rather than each graph, i.e., we only constrain the total perturbations over all graphs rather than the perturbations on each graph. This can greatly reduce the computational cost in solving the above-constrained maximization problem in that it reduces the number of constraints from twice the batch size to 2. However, it also introduces the additional risk that the perturbations could be severely imbalanced among the graphs in the batch, e.g., a graph is heavily perturbed while others are almost unchanged. In our experiment, we do not observe this problem but it could theoretically happen. A good practice is to start from this simple form, and then gradually add constraints to the vulnerable graphs in the batch if the imbalanced perturbations are observed.

During the attack stage, the perturbation matrix $\mathbf{L}_{\mathbf{A}}$ and its convex hull $\Tilde{\mathbf{L}}_{\mathbf{A}}$ are further subject to the constraints that they should be block diagonal matrices with 0 at position $(i,j)$ if node $i$ and node $j$ are the nodes from two graphs in the batch. This could be easily implemented by using a block diagonal mask to zero out the gradients during the forward propagation,
\begin{align}
 \mathbf{L}_{\mathbf{A}}&=\text{block\_diag}(\mathbf{1}_{n_i\times n_i}|i=1\cdots,b)\circ\mathbf{L}_{\mathbf{A}},\\
  \tilde{\mathbf{L}}_{\mathbf{A}}&=\text{block\_diag}(\mathbf{1}_{n_i\times n_i}|i=1\cdots,b)\circ \Tilde{\mathbf{L}}_{\mathbf{A}},
\end{align}
where $n_i$ is the number of nodes in the graph $G_i$ in the batch. With this processing, the projection operation on the adjacency matrix remains the same as in Equation (\ref{eq:proj}). In case we need to apply the constraints for each graph in the batch, we just need to apply the projection operation defined in Equation (\ref{eq:proj}) on the adjacency matrix of each graph, using the bisection method to solve $\mu$ for each graph separately. The projection operation on the feature perturbation matrix is not affected on the graph level, which still clips $\mathbf{L}_{\mathbf{X}}$ into the range $[-\delta_{\mathbf{X}},\delta_{\mathbf{X}}]$ elementwisely.

Furthermore, the information regularization also applies to graph-level contrastive learning, where we only need to replace the node embedding with the graph embedding in Equation (\ref{eq:ir}). Hence, we can derive the bound atop different views of the same graph in $\mathcal{B}$, $\mathcal{B}^+$ and $\mathcal{B}_{\text{adv}}$,
\begin{align}
    &d_i = 2\cdot\theta(R(\mathbf{H}^+_i), R(\mathbf{H}_{\text{adv},i})) -(\theta(R(\mathbf{H}^+_i), R(\mathbf{H}_i))+\theta(R(\mathbf{H}_{\text{adv,i}}), R(\mathbf{H}_i))),\\
    &L_{I}(\mathcal{B}, \mathcal{B}^+, \mathcal{B}_{\text{adv}}) = \frac{1}{b}\sum\limits_{i=1}^b{\max\{d_i, 0\}}.
\end{align}

The final loss of \ariel\ for the graph-level contrastive learning could be written as
\begin{equation}
\label{eq:ARIEL1}
    L(\mathcal{B},\mathcal{B}^+, \mathcal{B}_{\text{adv}}) = L_{\text{con}}(\mathcal{B},\mathcal{B}^+)+\epsilon_1 L_{\text{con}}(\mathcal{B}^+,\mathcal{B}_{\text{adv}}) + \epsilon_2  L_{I}(\mathcal{B}, \mathcal{B}^+, \mathcal{B}_{\text{adv}}).
\end{equation}
The graph-level adversarial contrastive learning could also follow the steps outlined in Algorithm \ref{alg:algorithm} for training, by simply replacing the input graph with the input batch in loss functions.

\section{Experiments}
In this section, we conduct empirical evaluations, which are designed to answer the following three questions:
\begin{itemize}
    \item [RQ1.] How effective is the proposed \ariel\ in comparison with previous graph contrastive learning methods on the node classification and graph classification task?
    \item [RQ2.] To what extent does \ariel\ gain robustness over the attacked graph?
    \item [RQ3.] How does each part of \ariel\ contribute to its performance?
\end{itemize}
We evaluate our method with the node classification task and graph classification task on the real-world graphs and further evaluate the robustness of it with the node classification task on the attacked graphs. The node/graph embeddings are first learned by the proposed \ariel\ algorithm, then the embeddings are fixed to perform the classification with a simple classifier trained over it. All our experiments are conducted on the NVIDIA Tesla V100S GPU with 32G memory.

\subsection{Experimental Setup}

\subsubsection{Datasets} For the node-level contrastive learning, we use eight datasets for the evaluation, including \textit{Cora}, \textit{CiteSeer}, \textit{Amazon-Computers}, \textit{Amazon-Photo}, \textit{Coauthor-CS}, \textit{Coauthor-Physics}, \textit{Facebook} and \textit{LastFM Asia}.
Cora and CiteSeer \cite{yang2016revisiting} are citation networks, where nodes represent documents and edges correspond to citations. Amazon-Computers and Amazon-Photo \cite{shchur2019pitfalls} are extracted from the Amazon co-purchase graph. In these graphs, nodes are the goods and they are connected by an edge if they are frequently bought together. Coauthor-CS and Coauthor-Physics \cite{shchur2019pitfalls} are the co-authorship graphs, where each node is an author and the edge indicates the co-authorship on a paper. Facebook \cite{https://doi.org/10.48550/arxiv.1909.13021} is a page-page graph of verified Facebook pages where edges correspond to the likes of each other. LastFM Asia \cite{https://doi.org/10.48550/arxiv.2005.07959} is a social network of Asian users, each node represents a user and they are connected via friendship.
 
For the graph-level contrastive learning, we evaluate \ariel\ on four datasets from the benchmark TUDataset \cite{https://doi.org/10.48550/arxiv.2007.08663}, including the biochemical molecules graphs NCI1, PROTEINS, DD and MUTAG.
 
A summary of the datasets\footnote{All the datasets are from Pytorch Geometric 2.0.4: \url{https://pytorch-geometric.readthedocs.io/en/latest/modules/datasets.html}} statistics is in Table \ref{tab:dataset} and Table \ref{tab:dataset1}.
\begin{table}
\centering
\begin{tabular}{lcccc}
\toprule
Dataset & Nodes & Edges & Features & Classes \\
\midrule
Cora       & 2,708  & 5,429 &1,433 &7    \\
CiteSeer    &3,327  & 4,732 &3,703 &6     \\
Amazon-Computers & 13,752 & 245,861&767&10\\
Amazon-Photo & 7,650 &119,081&745&8\\
Coauthor-CS & 18,333&81,894&6,805&15\\
Coauthor-Physics&34,493&247,962&8,415&5\\
Facebook & 22,470 &342,004 &128 & 4\\
LastFM Asia & 7,624 &55,612  &128&18\\ 
\bottomrule
\end{tabular}
\vspace{5pt}
\caption{Node-level contrastive learning datasets statistics, the number of nodes, edges, node feature dimensions, and classes are listed.}
\label{tab:dataset}
\vspace{-10pt}
\end{table}

\begin{table}
\centering
\begin{tabular}{lccccc}
\toprule
Dataset & Graphs &  Nodes &  Degree & Features & Classes \\
\midrule
NCI1    & 4110  & 29.87  & 1.08 &37 &2   \\
PROTEINS& 1113    & 39.06  & 1.86 &3 &2     \\
DD & 1178&284.32 & 2.52&89&2\\
MUTAG & 188 & 17.93 &1.10&7&2\\
\bottomrule
\end{tabular}
\vspace{5pt}
\caption{Graph-level contrastive learning datasets statistics, number of graphs, the average number of nodes and degree, and the number of node feature dimensions and classes are listed.}
\label{tab:dataset1}
\vspace{-20pt}
\end{table}

\subsubsection{Baselines} We consider seven graph contrastive learning methods for node-level contrastive learning, including DeepWalk \cite{Perozzi:2014:DOL:2623330.2623732},  DGI \cite{velickovic2018deep},  Robust DGI (abbreviated as RDGI) \cite{Xu2020UnsupervisedAR}, GMI \cite{peng2020graph}, MVGRL \cite{hassani2020contrastive}, GRACE \cite{zhu2020deep} and GCA \cite{Zhu_2021}. Since DeepWalk only generates the embeddings for the graph topology, we concatenate the node features to the generated embeddings for evaluation so that the final embeddings can incorporate both topology and attribute information. Besides, we also compare our method with two supervised methods Graph Convolutional Network (GCN) \cite{kipf2017semisupervised} and Graph Attention Network (GAT) \cite{velickovic2018graph}. 

For graph-level contrastive learning, we compare \ariel\ with the state-of-the-art graph kernel methods including graphlet kernel (GL), Weisfeiler-Lehman sub-tree kernel (WL) and deep graph kernel (DGK), and recent unsupervised graph representation learning methods including node2vec \cite{grover2016node2vec},
sub2vec \cite{Adhikari2018Sub2VecFL}, graph2vec \cite{https://doi.org/10.48550/arxiv.1707.05005}, InfoGraph \cite{https://doi.org/10.48550/arxiv.1908.01000} and GraphCL \cite{you2020graph}.

\subsubsection{Evaluation protocol} For each dataset, we randomly select $10\%$ nodes/graphs for training, $10\%$ nodes/graphs for validation, and the remaining for testing. For contrastive learning methods, a logistic regression classifier is trained to do the node classification over the node embeddings while a support vector machine is trained to do the graph classification over the graph embeddings. The accuracy is used as the evaluation metric. 

For node-level contrastive learning, we search each method over $6$ different random seeds, including  $5$ random seeds from our own and the best random seed of GCA on each dataset. For each seed, we evaluate the method on $20$ random training-validation-testing dataset splits and report the mean and the standard deviation of the accuracy on the best seed. Specifically, for the supervised learning methods, we abandon the existing splits, for example on Cora and CiteSeer, but instead do a random split before the training and report the results over $20$ splits.

For graph-level contrastive learning, we keep the evaluation protocol the same as the setting in \cite{https://doi.org/10.48550/arxiv.1908.01000} and \cite{you2020graph}, where the experiments are conducted on 5 random seeds, each corresponding to a 10-fold evaluation.

Besides testing on the original, clean graphs, we also evaluate our method on the attacked graphs for node-level contrastive learning. We use Metattack \cite{zugner2018adversarial} to perform the poisoning attack. 
Since Metattack is targeted at graph structure only and computationally inefficient on large graphs, we first randomly sample a subgraph of $5000$ nodes if the number of nodes in the original graph is greater than $5000$, then we randomly mask out $20\%$ node features and finally use Metattack to perturb $20\%$ edges to generate the final attacked graph.
For \ariel, we use the hyperparameters of the best models we obtain on the clean graphs for evaluation. For GCA, we report the performance in our main results for its three variants, GCA-DE, GCA-PR, and GCA-EV, which correspond to the adoption of degree, eigenvector, and PageRank \cite{ilprints422, kang2018aurora} centrality measures,  and use the best variant on each dataset for the evaluation on the attacked graphs.
\subsection{Hyperparameters}
For node-level contrastive learning, we use the same parameters and design choices for \ariel's network
architecture, optimizer and training scheme as in GRACE and GCA on each dataset. However, we find GCA not behave well on Cora with a significant performance drop, so we re-search the parameters for GCA on Cora separately and use a different temperature for it. Other contrastive learning-specific parameters are kept the same over GRACE, GCA and \ariel.  
On graph-level contrastive learning, we keep \ariel's hyperparameters the same as the ones used by GraphCL except for its own parameters.

All GNN-based baselines on node-contrastive learning use a two-layer GCN as the encoder. For each method, we compare its default hyperparameters and the ones used by \ariel\ and use the hyperparameters leading to better performance. Other algorithm-specific hyperparameters all respect the default setting in its official implementation. For graph-level contrastive learning, \ariel\ uses a three-layer GIN as the encoder, and we take the results for each baseline from its original paper under the same experimental setting. 

Other hyperparameters of \ariel\ are summarized  as follows:
\begin{itemize}
    \item Adversarial contrastive loss coefficient $\epsilon_1$ and information regularization strength $\epsilon_2$. We search them over $\{0.5, 1, 1.5, 2\}$ and use the one with the best performance on the validation set 
    of each dataset. Specifically, we first fix $\epsilon_2$ as $0$ and decide the optimal value for all other parameters, then we search $\epsilon_2$ on top of the model with other hyperparameters fixed.
    \item Number of attack steps and perturbation constraints. These parameters are fixed on all datasets. For node-level contrastive learning, we set the number of attack steps 5, edge perturbation constraint $\Delta_{\mathbf{A}}=0.1\sum_{i,j}\mathbf{A}[i,j]$ and feature perturbation constraint $\delta_{\mathbf{X}}=0.5$. For graph-level contrastive learning, we set the number of attack steps 5, edge perturbation constraint $\Delta_{\mathbf{A}}=0.05\sum_{i,j}\mathbf{A}[i,j]$ and feature perturbation constraint $\delta_{\mathbf{X}}=0.04$.
    \item Curriculum learning weight $\gamma$ and change period $T$. In our experiments, we simply fix $\gamma=1.1$ and $T=20$ for node-level contrastive learning and $\gamma=1$ for graph-level contrastive learning.
    \item Graph perturbation rate $\alpha$ and feature perturbation rate $\beta$. We search both of them over $\{0.001, 0.01, 0.1\}$ and take the best one on the validation set of each dataset.
    \item Subgraph size. On node-level contrastive learning, we keep the subgraph size $500$ for \ariel\ on all datasets except Facebook and LastFM Asia, where we use a subgraph size $3000$. We do not do the subgraph sampling on graph-level contrastive learning. Instead, we control the batch size $b$, where we fix $b=32$ for DD and $b=128$ for the other three datasets.
\end{itemize}

\subsection{Main Results}
The comparison results of node classification on all eight datasets are summarized in Table \ref{tab:results}. 
Our method \ariel\ outperforms baselines over all datasets except on Cora and Facebook, with only $0.11\%$ and $0.40\%$ lower in accuracy than MVGRL. It can be seen that the previous state-of-the-art method GCA does not bear significant improvements over previous methods. In contrast, \ariel\ can achieve consistent improvements over GRACE and GCA on all datasets, especially on Amazon-Computers with almost $3\%$ gain. 
\begin{table}[h]
\centering
\resizebox{\textwidth}{!}{%
\begin{tabular}{ccccccccc}
\toprule
\textbf{Method}  & \textbf{Cora} & \textbf{CiteSeer} & \textbf{\makecell{Amazon-\\Computers}} & \textbf{\makecell{Amazon-\\Photo}} & \textbf{\makecell{Coauthor-\\CS}} & \textbf{\makecell{Coauthor-\\Physics}}&\textbf{Facebook} &\textbf{\makecell{LastFM \\Asia}} \\
\midrule
GCN &$84.14\pm0.68$&$69.02\pm0.94$&$88.03\pm1.41$&$92.65\pm0.71$&$92.77\pm0.19$&$95.76\pm0.11$ &89.98 $\pm$ 0.26 & 83.96 $\pm$ 0.47\\
GAT&$83.18\pm1.17$&$69.48\pm1.04$&$85.52\pm2.05$&$91.35\pm1.70$&$90.47\pm0.35$&$94.82\pm 0.21$&89.97 $\pm$ 0.39 & 83.04 $\pm$ 0.39 \\
\midrule
DeepWalk&$79.82\pm0.85$&$67.14\pm0.81$&$86.23\pm0.37$&$90.45\pm0.45$&$85.02\pm0.44$&$94.57\pm0.20$& 86.67 $\pm$ 0.22 & 83.93 $\pm$ 0.61\\
DGI     &$84.24\pm0.75$   & $69.12\pm1.29$ & $88.78\pm0.43$   & $92.57\pm0.23$&$92.26\pm0.12$&$95.38\pm0.07$& 89.80 $\pm$ 0.27 & 82.88 $\pm$ 0.52\\
RDGI & $81.84\pm1.07$ & $65.92\pm 1.26$ & $88.07\pm 0.28$ & $92.17\pm0.27$ & OOM & OOM & OOM & $77.34\pm0.69$\\
GMI &$82.43\pm0.90$&$70.14\pm1.00$&$83.57\pm0.40$&$88.04\pm0.59$&OOM&OOM & OOM & 74.71 $\pm$ 0.70\\
MVGRL   &\textbf{84.39 $\pm$ 0.77}   &$69.85\pm1.54$& $89.02\pm0.21$&$92.92\pm0.25$ & $92.22\pm0.22$&$95.49\pm0.17$& \textbf{90.60 $\pm$ 0.28} & 83.83 $\pm$ 0.85\\

GRACE   &$83.40\pm1.08$ & $69.47\pm1.36$  &  $87.77\pm0.34$ & $92.62\pm0.25$ &$93.06\pm 0.08$ &$95.64\pm0.08$ & 88.95 $\pm$ 0.31 & 79.52 $\pm$ 0.64\\
GCA-DE &$82.57\pm0.87$& $72.11\pm0.98$ & $88.10\pm0.33$ &$92.87\pm0.27$ &$93.08\pm0.18$&$95.62\pm0.13$ &89.73 $\pm$ 0.37& 82.42 $\pm$ 0.46\\
GCA-PR &$82.54\pm0.87$& $72.16\pm0.73$ & $88.18\pm0.39$ &$92.85\pm0.34$ &$93.09\pm0.15$&$95.58\pm0.12$ &  89.68 $\pm$ 0.36  &82.44 $\pm$ 0.51\\
GCA-EV &$81.80\pm0.92$& $67.07\pm0.79$ & $87.95\pm0.43$ &$92.63\pm0.33$ &$93.06\pm0.14$&$95.64\pm0.08$ & 89.68 $\pm$ 0.38& 82.35 $\pm$ 0.46\\
\midrule
\textbf{ARIEL}  & 84.28 $\pm$ 0.96 & \textbf{72.74 $\pm$ 1.10} &\textbf{91.13 $\pm$ 0.34} &\textbf{94.01 $\pm$ 0.23}&\textbf{93.83 $\pm$ 0.14}&\textbf{95.98 $\pm$ 0.05}   & 90.20$\pm$ 0.23 &\textbf{84.04$\pm$0.44}  \\
\bottomrule
\end{tabular}}
\vspace{5pt}
\caption{Node classification accuracy in percentage on eight real-world datasets. We bold the results with the best mean accuracy. The methods above the line are the supervised ones, and the ones below the line are unsupervised. OOM stands for Out-of-Memory on our 32G GPUs.}
\label{tab:results}
\vspace{-10pt}
\end{table}

Besides, we find MVGRL a solid baseline whose performance is close to or even better than GCA on these datasets. It achieves the highest score on Cora and Facebook, and the second-highest on Amazon-Computers and Amazon-Photo. However, it does not behave well on CiteSeer, where GCA can effectively increase the score of GRACE. To sum up, previous modifications over the grounded frameworks are mostly based on specific knowledge, for example, MVGRL introduces the diffusion matrix to DGI and GCA defines the importance on the edges and features, and they cannot consistently take effect on all datasets. However, \ariel\ uses the adversarial attack to automatically construct the high-quality contrastive samples and achieves more stable performance improvements. 

In comparison with the supervised methods, \ariel\ also achieves a clear advantage over all of them. Although it would be premature to conclude that \ariel\ is more powerful than these supervised methods since they are usually tested under the specific training-testing split, these results do demonstrate that \ariel\ can indeed generate highly expressive node embeddings for the node classification task, which can achieve comparable performance to the supervised methods.

The graph classification results are summarized in Table \ref{tab:results1}.
\begin{table*}
\small
\centering
\begin{tabular}{ccccc}
\toprule
\textbf{Method}  & \textbf{NCI1} & \textbf{PROTEINS} & \textbf{DD} & \textbf{MUTAG}  \\
\midrule
GL&-&-&-& 81.66 $\pm$ 2.11\\
WL&80.11 $\pm$ 0.50&  72.92 $\pm$ 0.56 & - & 80.72 $\pm$ 3.00\\
DGK& \textbf{80.31 $\pm$ 0.46} & 73.30 $\pm$ 0.82 & - & 87.44 $\pm$ 2.72 \\
\midrule
node2vec& 54.89 $\pm$ 1.61 & 57.49 $\pm$ 3.57 & - & 72.63 $\pm$ 10.20\\
sub2vec& 52.84 $\pm$ 1.47 & 53.03 $\pm$ 5.55 & - & 61.05 $\pm$ 15.80 \\
graph2vec& 73.22 $\pm$  1.81 & 73.30 $\pm$  2.05 & - & 83.15 $\pm$  9.25\\
InfoGraph&  76.20 $\pm$ 1.06 & 74.44 $\pm$ 0.31 & 72.85 $\pm$ 1.78 & 89.01 $\pm$ 1.13 \\
GraphCL & 77.87 $\pm$ 0.41&  74.39 $\pm$ 0.45 & 78.62 $\pm$ 0.40& 86.80 $\pm$ 1.34 \\
\midrule
\textbf{ARIEL}  & 78.91 $\pm$ 0.36& \textbf{75.22 $\pm$ 0.26} &\textbf{79.15 $\pm$ 0.53}& \textbf{89.25 $\pm$ 1.18}  \\
\bottomrule
\end{tabular}
\vspace{5pt}
\caption{Graph classification accuracy in percentage on four real-world datasets. We bold the results with the best mean accuracy. The methods above the double line belong to the graph kernel methods, and the ones below the double line are unsupervised representation learning methods. The compared numbers are from the original paper under the same experimental setting. }
\label{tab:results1}
\vspace{-20pt}
\end{table*}
Compared with our basic framework GraphCL, which uses naive augmentation methods, \ariel\ achieves even stronger performance on all datasets. GraphCL does not show a clear advantage against previous baselines such as InfoGraph and it does not behave well on the dataset with small graph size (e.g., NCI1 and MUTAG). However, \ariel\ can take the lead on three of the datasets and greatly reduce the performance gap on NCI1 with the graph kernel methods. It can be clearly seen that \ariel\ behaves better than GraphCL on NCI1 and MUTAG with at least $1\%$ improvement in accuracy. In comparison with another graph contrastive learning method InfoGraph, we can also see that \ariel\ takes an overall lead on all datasets, even on MUTAG where InfoGraph shows a dominant advantage against other baselines.

The above empirical results on the node classification and graph classification tasks clearly demonstrate the advantage of \ariel\ on real-world graphs, which indicates the better augmentation strategy of \ariel.

\subsection{Results under Attack}
 The results on attacked graphs are summarized in Table \ref{tab:att_results}. Specifically, we evaluate all these methods on the attacked subgraph of Amazon-Computers, Amazon-Photo, Coauthor-CS, Coauthor-Physics, Facebook, and LastFM Asia, so their results are not directly comparable to the results in Table \ref{tab:results}. To compare with the previous results, we look at the datasets where \ariel\ takes the lead, and then find the performance of the second-best method on each dataset for both the original graph and the attacked one. If \ariel\ outperforms the second-best method by a much larger margin on the attacked graph compared with that on the original graph, we claim that \ariel\ is significantly robust on that dataset. 
 
\begin{table}[h]
\centering
\resizebox{\textwidth}{!}{%
\begin{tabular}{ccccccccc}
\toprule
\textbf{Method}  & \textbf{Cora} & \textbf{CiteSeer} & \textbf{\makecell{Amazon-\\Computers}} & \textbf{\makecell{Amazon-\\Photo}} & \textbf{\makecell{Coauthor-\\CS}} & \textbf{\makecell{Coauthor-\\Physics}}&\textbf{Facebook} &\textbf{\makecell{LastFM \\Asia}}  \\
\midrule
GCN &$80.03\pm0.91$&$62.98\pm1.20$& $84.10\pm1.05$ & $91.72\pm0.94$ & $80.32\pm0.59$ & $87.47\pm0.38$ & 70.07 $\pm$ 0.74 & 73.22 $\pm$ 0.85 \\
GAT &$79.49\pm1.29$&$63.30\pm1.11$&$81.60\pm1.59$&$90.66\pm 1.62$ & $77.75\pm0.80$& $86.65\pm0.41$&\textbf{72.02 $\pm$ 0.78} & 73.21 $\pm$ 0.64\\
\midrule
DeepWalk&$74.12\pm1.02$&$63.20\pm0.80$&$79.08\pm0.67$&$88.06\pm0.41$& $49.30\pm 1.23$ & $79.26\pm1.38$& 59.07 $\pm$ 1.01 & 67.61 $\pm$ 0.80\\
DGI &$80.84\pm0.82$&$64.25\pm0.96$ & $83.36\pm 0.55$ & $91.27\pm0.29$&$78.73\pm0.50$ & $85.88\pm0.37$& 70.52 $\pm$ 0.93 & 71.80 $\pm$ 0.59   \\
RDGI & $77.29\pm1.01$& $59.94\pm1.29$  &$82.35\pm0.59$& $90.63\pm 0.41$ & $83.09\pm0.64$&$83.58\pm 0.75$&$67.85\pm 1.19$&$63.59\pm0.91$\\
GMI &79.17 $\pm$ 0.76&$65.37\pm1.03$& $77.42\pm0.59$& $89.44\pm0.47$& $80.92\pm0.64$ & $87.72\pm0.45$ & 68.93 $\pm$ 0.83 & 58.89 $\pm$ 0.95\\
MVGRL & \textbf{80.90 $\pm$ 0.75}&$64.81\pm1.53$ & $83.76\pm0.69$ & $91.76\pm0.44$ & $79.49 \pm 0.75$ & $86.98\pm0.61$& 71.76 $\pm$ 0.69 & 73.42 $\pm$ 1.11\\
GRACE &$78.55\pm0.81$ & $63.17\pm 1.81$ & $84.74\pm1.13$ & $91.26\pm0.37$ &$80.61\pm0.63$ & $85.71\pm0.38$ &71.97 $\pm$ 0.98 &69.39 $\pm$ 0.63\\
GCA &$76.79\pm0.97$&$64.89\pm1.33$& $85.05\pm0.51$ &$91.71\pm0.34$&$82.72\pm0.58$&$89.00\pm0.31$& 69.54 $\pm$ 0.82 & 71.83 $\pm$ 1.03\\
\midrule
\textbf{ARIEL}  &  80.33 $\pm$ 1.25 & \textbf{69.13 $\pm$ 0.94} & \textbf{88.61 $\pm$ 0.46} & \textbf{92.99 $\pm$ 0.21} & \textbf{84.43 $\pm$ 0.59} & \textbf{89.09 $\pm$ 0.31} & 71.15 $\pm$ 1.19 & \textbf{73.94 $\pm$  0.78}\\
\bottomrule
\end{tabular}}
\vspace{5pt}
\caption{Node classification accuracy in percentage on the graphs under Metattack, where subgraphs of Amazon-Computers, Amazon-Photo, Coauthor-CS and Coauthor-Physics, Facebook and LastFM Asia are used for attack and their results are not directly comparable to those in Table \ref{tab:results}. We bold the results with the best mean accuracy. GCA is evaluated on its best variant on each clean graph.}
\label{tab:att_results}
\vspace{-20pt}
\end{table}
 
 Under this principle, we can see that \ariel\ is significantly robust on CiteSeer, with the margin to the second best method increasing from $0.58\%$ to $3.96\%$, Amazon-Computers, with the margin increasing from $2.11\%$ to $3.56\%$, and Coauthor-CS, with the margin increasing from $0.74\%$ to $1.71\%$. On Coauthor-Physics, \ariel\ and GCA both show clear robustness against the remaining methods.  Although some baselines are robust on specific datasets, for example, MVGRL on Cora, GMI on CiteSeer, GCN on Facebook, and GCA on Coauthor-CS and Coauthor-Physics, they fail to achieve consistent robustness over all datasets. Although GCA indeed makes GRACE more robust for most datasets, it is still vulnerable on Cora, CiteSeer, and Amazon-Computers, with more than $3\%$ lower than \ariel\ in the final accuracy. 
 
 Besides, we can also see that \ariel\ still shows high robustness on the datasets where it cannot take the lead. On Cora and Facebook, \ariel\ is only less than $1\%$ lower in accuracy than the best method and it is still better than most baselines. It does not show a sudden performance drop on any dataset, such as MVGRL on CiteSeer and GCA on Facebook.

Basically, MVGRL and GCA can improve the robustness of their respective grounded frameworks over specific datasets, but we find this kind of improvement relatively minor. Instead, \ariel\ has more significant improvements and greatly increases robustness. It is worth noting that though RDGI is also developed to improve the robustness of graph representation learning, it does not show a clear advantage against DGI in our evaluation. This is mainly because the original RDGI considers the attack at test time and what we evaluate is the robustness against the attack at training time, which is more common for the graph learning tasks \cite{bojchevski2019adversarial, Z_gner_2018, zugner2018adversarial}. Based on the comparative results, we claim that \ariel\ is more robust than previous graph contrastive learning methods in the face of the adversarial attack.

\subsection{Ablation Study}
For this section, we first set $\epsilon_2$ as $0$ and investigate the role of adversarial contrastive loss. The adversarial contrastive loss coefficient $\epsilon_1$ controls the portion of the adversarial contrastive loss in the final loss. When $\epsilon_1=0$, the final loss reduces to the regular contrastive loss in Equation (\ref{eq:contra}). To explore the effect of the adversarial contrastive loss, we fix other parameters in our best models on Cora and CiteSeer and gradually increase $\epsilon_1$ from $0$ to $2$. The changes in the final performance are shown in Figure \ref{fig:eps1}.

The dashed line represents the performance of GRACE with subgraph sampling, i.e., $\epsilon_1=0$. Although there exist some variations, \ariel\ is always above the baseline under a positive $\epsilon_1$ with around $2\%$ improvement. The subgraph sampling trick may sometimes help the model, for example, it improves GRACE without subgraph sampling by $1\%$ on CiteSeer, but it could be detrimental as well, such as on Cora. This is understandable since subgraph sampling can simultaneously enrich the data augmentation and lessen the number of negative samples, both critical to contrastive learning. While for the adversarial contrastive loss, it has a stable and significant improvement on GRACE with subgraph sampling, which demonstrates that the performance improvement of \ariel\ mainly stems from the adversarial loss rather than the subgraph sampling.
\begin{figure}[h]
    \begin{minipage}{0.4\linewidth}
    \includegraphics[width=\linewidth]{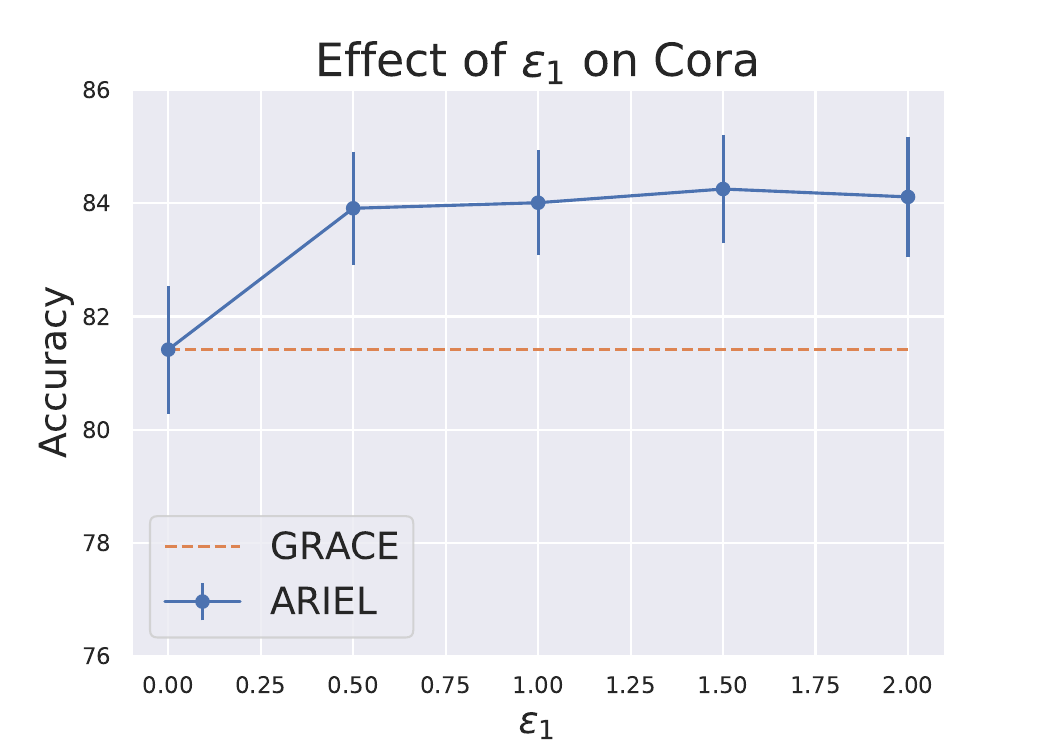}
    \vspace{-15pt}
    \end{minipage}
    \begin{minipage}{.4\linewidth}
    \includegraphics[width=\linewidth]{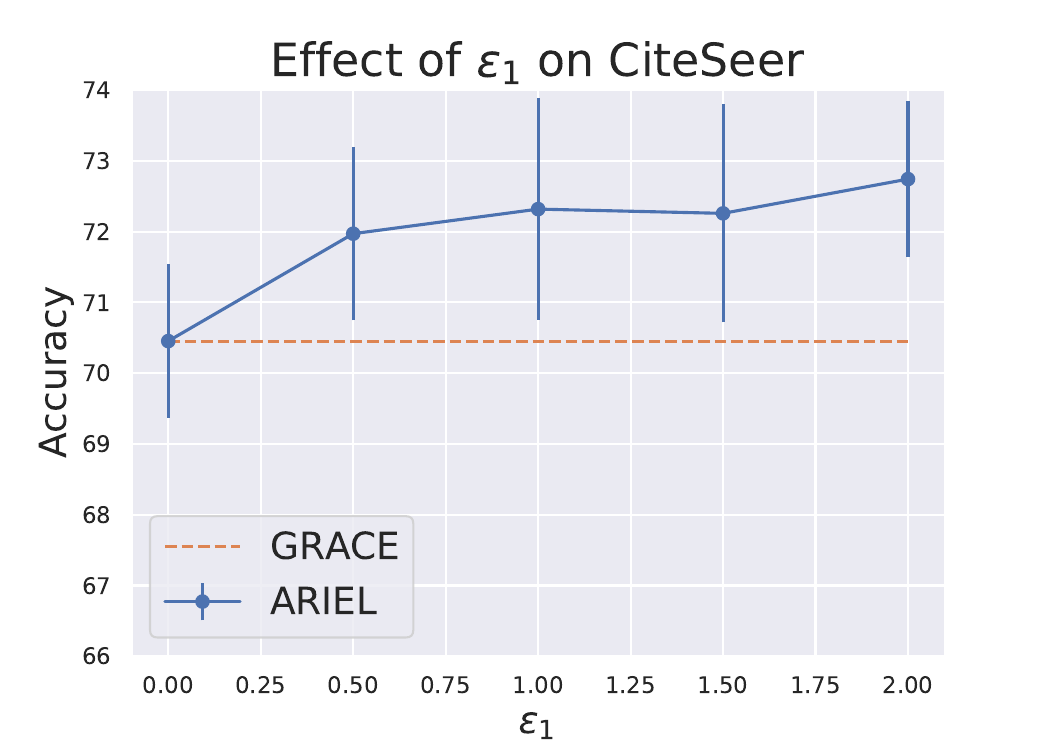}
    \end{minipage}
    \caption{Effect of adversarial contrastive loss coefficient $\epsilon_1$ on Cora and CiteSeer. The dashed line represents the performance of GRACE with subgraph sampling.}
    \label{fig:eps1}
    \vspace{-15pt}
\end{figure}

Next, we fix all other parameters and check the behavior of $\epsilon_2$. Information regularization is mainly designed to stabilize the training of \ariel. We find \ariel\ would experience the collapsing at the early training stage and the information regularization could mitigate this issue. We choose the best run on Amazon-Photo, where the collapsing frequently occurs, and similar to before, we gradually increase $\epsilon_2$ from $0$ to $2$, the results are shown in Figure \ref{fig:eps2} (left).
\begin{figure}[h]
    \begin{minipage}{0.4\linewidth}
    \includegraphics[width=\linewidth]{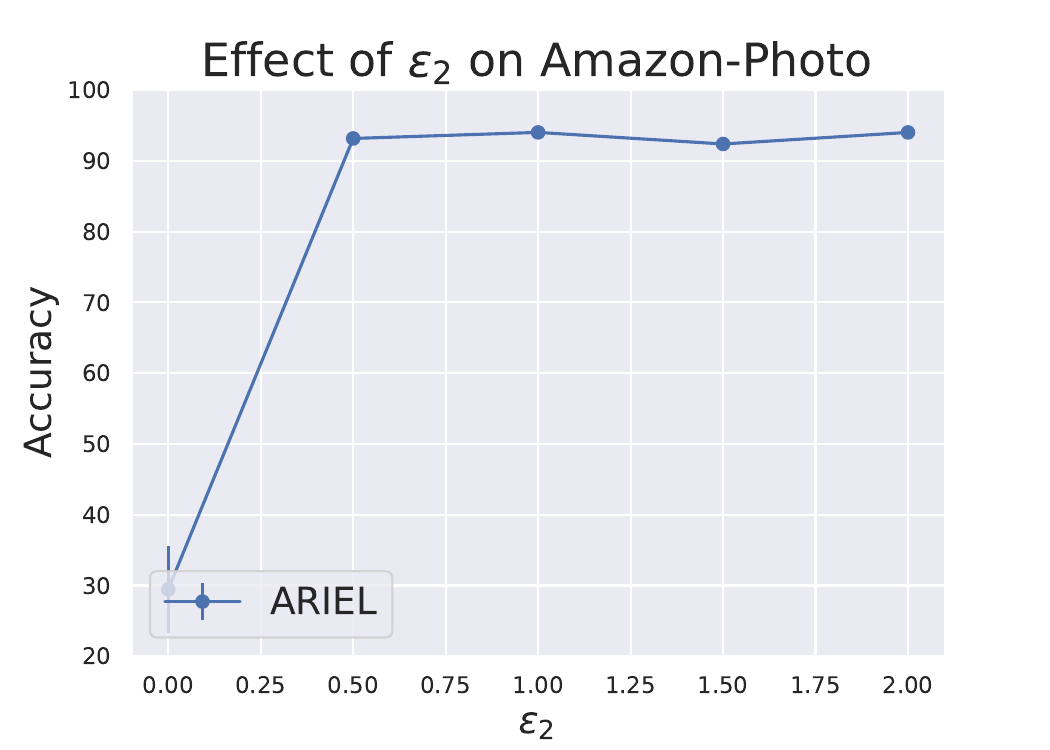}
    \end{minipage}
    \begin{minipage}{.4\linewidth}
    \includegraphics[width=\linewidth]{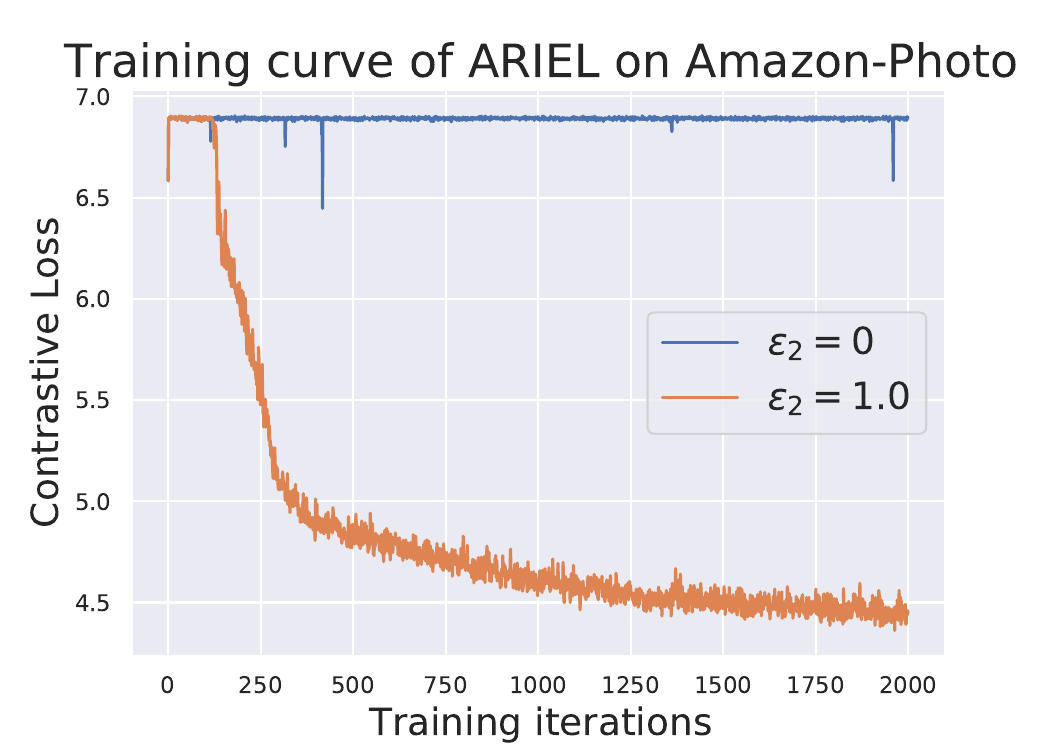}
    \end{minipage}
    \caption{Effect of information regularization on Amazon-Photo. The left figure shows the model performance under different $\epsilon_2$ and the right figure plots the training curve of \ariel\ under $\epsilon_2=0$ and $\epsilon_2=1.0$.}
    \label{fig:eps2}
    \vspace{-10pt}
\end{figure}
As can be seen, without using the information regularization, \ariel\ could collapse without learning anything, while setting $\epsilon_2$ greater than $0$ can effectively avoid this situation. To further illustrate this, we draw the training curve of the regular contrastive loss in Figure \ref{fig:eps2} (right), for the best \ariel\ model on Amazon-Photo and the same model by simply removing the information regularization. Without information regularization, the model could get stuck in a bad parameter area and fail to converge, while information regularization can resolve this issue.

\subsection{Training Analysis}
Here we compare the training of \ariel\ on node-level contrastive learning to other methods on our NVIDIA Tesla V100S GPU with 32G memory.

Adversarial attacks on graphs tend to be highly computationally expensive since the attack requires the gradient calculation over the entire adjacency matrix, which is of size $O(n^2)$. For \ariel, we resolve this bottleneck with subgraph sampling on large graphs and empirically show that the adversarial training on the subgraph still yields significant improvements, without increasing the number of training iterations. In our experiments, we find GMI the most memory inefficient, which cannot be trained on Coauthor-CS, Coauthor-Physics, and Facebook. For DGI, MVGRL, GRACE, and GCA, the training of them also amounts to $30$G GPU memory on Coauthor-Physics while the training of \ariel\ requires no more than $8$G GPU memory. In terms of the training time, DGI and MVGRL are the fastest to converge, but it takes MVGRL a long time to compute the diffusion matrix on large graphs. \ariel\ is slower than GRACE and GCA on Cora and CiteSeer, but it is faster on large graphs like  Coauthor-CS and Coauthor-Physics, with the training time for each iteration invariant to the graph size due to the subgraph sampling. On the largest graph Coauthor-Physics, each iteration takes GRACE 0.875 second and GCA 1.264 seconds, while it only takes \ariel\ 0.082 second. This demonstrates that \ariel\ has even better scalability than GRACE and GCA.

Subgraph sampling, under some mild assumptions, could be an efficient way to reduce the computational cost for any node-level contrastive learning algorithm. Besides this general trick, we want to point out that the attack is in fact not always needed on the whole graph to generate the adversarial view. Another solution to avoid explosive memory is to select some anchor nodes and only perturb the edges among these anchor nodes and their features. Since the scalability issue has been resolved by subgraph sampling on all datasets appearing in this work, we will not further discuss the details of this method and empirically prove its effectiveness, but leave it for future work.

\section{Related Work}
In this section, we review the related work in the following three categories: graph contrastive learning, adversarial attack on graphs, and adversarial contrastive learning.

\subsection{Graph Contrastive Learning}
Contrastive learning is known for its simplicity and strong expressivity. Traditional methods ground the contrastive samples on the node proximity in the graph, such as DeepWalk \cite{Perozzi:2014:DOL:2623330.2623732} and node2vec \cite{grover2016node2vec}, which use random walks to generate the node sequences and approximate the co-occurrence probability of node pairs. However, these methods can only learn the embeddings for the graph structures but regardless of the node features. 

GNNs \cite{kipf2017semisupervised, velickovic2018graph} can easily capture the local proximity and node features \cite{velickovic2018deep,zhao2020data,you2020graph,jing2021multiplex,jing2021network}. To further improve the performance, the Information Maximization (InfoMax) principle \cite{linsker1988self} has been introduced. DGI \cite{velickovic2018deep} is adapted from Deep InfoMax \cite{hjelm2019learning} to maximize the mutual information between the local and global features. It generates the negative samples with a corrupted graph and contrasts the node embeddings with the original graph embedding and the corrupted one. Based on a similar idea, GMI \cite{peng2020graph} generalizes the concept of mutual information to the graph domain by separately defining the mutual information on the features and edges. Graph Community Infomax (GCI) \cite{10.1145/3480244} instead tries to maximize the mutual information between the community representation and the node representation for those positive pairs. Another follow-up work of DGI, MVGRL \cite{hassani2020contrastive}, maximizes the mutual information between the first-order neighbors and graph-diffusion. On the graph level, InfoGraph \cite{https://doi.org/10.48550/arxiv.1908.01000} makes use of a similar idea to maximize the mutual information between the global representation and patch representation from the same graph.
HDI \cite{jing2021hdmi} introduces high-order mutual information to consider both intrinsic and extrinsic training signals.
However, mutual information-based methods generate the corrupted graphs by simply randomly shuffling the node features.
Recent methods exploit the graph topology and features to generate better-augmented graphs. GCC \cite{Qiu_2020} adopts a random-walk-based strategy to generate different views of the context graph for a node, but it ignores the augmentation on the feature level. GCA \cite{zhu2020deep}, instead, considers the data augmentation from both the topology and feature level, and introduces the adaptive augmentation by considering the importance of each edge and feature dimension. To investigate the power of different data augmentations in graph domains, GraphCL \cite{you2020graph} systematically studies the different combinations of graph augmentation strategies and applies them to different graph learning settings. 
Unlike the above methods which construct the data augmentation samples based on domain knowledge, \ariel\ uses an adversarial attack to construct the view that maximizes the contrastive loss, which is more informative with broader applicability.

\subsection{Adversarial Attack on Graphs}
Deep learning methods are known vulnerable to adversarial attacks, and this is also the case in the graph domain. As shown by Bojchevski et al. \cite{bojchevski2019adversarial}, both random-walk-based methods and GNNs-based methods could be attacked by flipping a small portion of edges. Xu et al. \cite{xu2019topology} propose a PGD attack and min-max attack on the graph structure from the optimization perspective. NETTACK \cite{Z_gner_2018} is the first to attack GNNs using both structure attack and feature attack, causing a significant performance drop of GNNs on the benchmarks. After that, Metattack \cite{zugner2018adversarial}  formulates the poisoning attack of GNNs as a meta-learning problem and achieves remarkable performance by only perturbing the graph structure. Node Injection Poisoning Attacks \cite{NIPA} uses a hierarchical reinforcement learning approach to sequentially manipulate the labels and links of the injected nodes. Recently, InfMax \cite{ma2021adversarial} formulated the adversarial attack on GNNs as an influence maximization problem.

\subsection{Adversarial Contrastive Learning}
The concept of adversarial contrastive learning is first proposed on visual domains \cite{kim2020adversarial, jiang2020robust, ho2020contrastive}. All these works propose a similar idea to use the adversarial sample as a form of data augmentation in contrastive learning and it can bring better downstream task performance and higher robustness. ACL \cite{kim2020adversarial} studies the different paradigms of adversarial contrastive learning, by replacing one or two of the augmentation views with the adversarial view generated by PGD attack \cite{madry2019deep}. CLAE \cite{ho2020contrastive} and RoCL \cite{jiang2020robust} use FGSM \cite{goodfellow2015explaining} to generate an additional adversarial view atop the two standard augmentation views. RDGI \cite{Xu2020UnsupervisedAR} and AD-GCL  \cite{suresh2021adversarial} are the most relevant work to ours in graph domains. RDGI quantifies the robustness of node representation as the decrease in mutual information between the graph and its embedding under adversarial attacks. It learns a robust node representation by simultaneously minimizing the standard contrastive learning loss and improving the robustness. Nonetheless, its objective sacrifices the expressiveness of the node representation for robustness while \ariel\ can improve both of them. AD-GCL formulates adversarial graph contrastive learning in a min-max form and uses a parameterized network for edge dropping. However, AD-GCL is designed for the graph classification task only and does not explore the robustness of graph contrastive learning. Finally, All previous adversarial contrastive learning methods do not take scalability into consideration, with visual models and AD-GCL dealing with independent instances and RDGI only working on small graphs, but \ariel\ can work for both interconnected instances (node embedding) and independent instances (graph embedding) on a large scale.
 
 Some recent theoretical analyses further reveal the vulnerability of contrastive learning. Jing et al. \cite{https://doi.org/10.48550/arxiv.2110.09348} show that dimensional collapse could happen if the variation of the data augmentation exceeds the variation of the data itself in contrastive learning. Wang et al. \cite{wang2022chaos} prove that contrastive learning could cluster the instances from the same class only when the support of different intra-class samples overlaps under data augmentation. The representations learned by contrastive learning may fail in downstream tasks when either under-overlapping or overly overlapping happens. From these perspectives, searching for adversarial contrastive samples in a safe area is more likely to generate useful representations for downstream tasks.

\section{Conclusion}
In this paper, we propose a universal framework for graph contrastive learning by introducing an adversarial view, scaling it through subgraph sampling, and stabilizing it through information regularization. It consistently outperforms the state-of-the-art graph contrastive learning methods in the node classification and graph classification tasks and exhibits a higher degree of robustness to the adversarial attack. Our framework is not limited to the graph contrastive learning frameworks we build on in this paper, and it can be naturally extended to other graph contrastive learning methods as well. In the future, we plan to further investigate (1) the adversarial attack on graph contrastive learning and (2) the integration of graph contrastive learning and supervised methods.

\section*{Acknowledgements}
This work is supported by NSF (1947135, 
2134079, 
2316233, 
and 2324770
),
the NSF Program on Fairness in AI in collaboration with Amazon (1939725), 
    DARPA (HR001121C0165), 
NIFA (2020-67021-32799), 
DHS (17STQAC00001-07-00), 
ARO (W911NF2110088), 
the C3.ai Digital Transformation Institute, 
MIT-IBM Watson AI Lab, 
and IBM-Illinois Discovery Accelerator Institute. 
The content of the information in this document does not necessarily reflect the position or the policy of the Government or Amazon, and no official endorsement should be inferred.  The U.S. Government is authorized to reproduce and distribute reprints for Government purposes notwithstanding any copyright notation here on.


\bibliographystyle{ACM-Reference-Format}
\bibliography{main}

\end{document}